\documentclass[11pt,draftclsnofoot,onecolumn]{IEEEtran}

\usepackage{amsmath,amsthm}
\usepackage{amssymb}
\usepackage{graphicx}
\usepackage{multirow}
\usepackage{amsfonts}
\usepackage{dsfont}

\usepackage{mathrsfs}
\usepackage{color}
\usepackage[font=small]{caption}
\usepackage{algorithm}
\usepackage{algorithmic}
\usepackage{amsthm}
\usepackage{amsmath}
\usepackage{amssymb}
\usepackage{graphicx}
\usepackage{multirow}
\usepackage{amsfonts}
\usepackage{dsfont}
\usepackage{url}
\usepackage{caption}
\usepackage{color}
\usepackage{setspace}
\usepackage[font=small]{caption}
\DeclareMathOperator*{\argmax}{arg\,max}

\ifCLASSINFOpdf
\else
\fi
\begin{document}
%
\title{Sparse Matrix-based Random Projection for Classification}
%
%
%

\author{
Weizhi~Lu,~Weiyu~Li,~Kidiyo~Kpalma and Joseph~Ronsin
}   

\maketitle

\begin{abstract}

As a typical dimensionality reduction technique, random projection can be simply implemented with linear projection, while maintaining the pairwise distances of high-dimensional data with high probability.  Considering this technique is mainly exploited for the task of  classification, this paper is  developed to study the construction of random  matrix from the viewpoint of feature selection,   rather than of  traditional distance preservation. This yields a somewhat surprising theoretical result, that is,  the sparse random matrix with exactly one nonzero element per column, can present better feature selection performance than other more dense   matrices, if the projection dimension is sufficiently large (namely, not much smaller than  the number of feature elements); otherwise, it will perform comparably to others. For random projection, this theoretical result  implies  considerable improvement  on both complexity and performance, which is widely confirmed with the classification experiments on both synthetic data and real data.


\end{abstract}

\begin{IEEEkeywords}
  Random Projection, Sparse Matrix,  Classification, Feature Selection, Distance Preservation, High-dimensional data
\end{IEEEkeywords}

%
\IEEEpeerreviewmaketitle

\theoremstyle{plain} \newtheorem{lemma}{Lemma}
\theoremstyle{plain} \newtheorem{theorem}{Theorem}
\theoremstyle{plain} \newtheorem{definition}{Definition}

\theoremstyle{plain} \newtheorem{corollary}[lemma]{Corollary}

\section{Introduction}

Random projection attempts to project a set of high-dimensional data into a low-dimensional subspace without  distortion on pairwise distance.  This brings attractive computational advantages on the collection and processing of high-dimensional signals. In practice, it has been successfully applied in numerous fields concerning categorization, as shown in \cite{Goel05} and the references therein. Currently the theoretical study of this technique mainly falls into one of the following two topics. One topic is concerned with the construction of random matrix in terms of distance preservation. In fact, this problem has been sufficiently addressed along with the emergence of Johnson-Lindenstrauss (JL) lemma \cite{Jonson84}.  The other popular one is to estimate the performance of traditional  classifiers combined with random projection, as detailed in \cite{Durrant13} and the references therein. Specifically, it may be worth mentioning that,  recently the performance consistency of SVM  on random projection is proved by exploiting the underlying connection between JL lemma and compressed sensing \cite{Calderbank09compressedlearning} \cite{Baraniuk08}.


Based on the principle of distance preservation,  Gaussian random matrices \cite{Indyk98} and a few  sparse $\{0, \pm1\}$  random  matrices \cite{Achlioptas03,Li06,Dasgupta10} have been sequentially proposed for random projection. In terms of implementation complexity, it is clear that the sparse random matrix is more attractive. Unfortunately, as it will be proved in the following section \ref{subsec-sparse},  the sparser matrix tends to yield weaker  distance preservation. This fact largely weakens our interests in the pursuit of sparser random matrix. However, it is necessary to mention a problem ignored for a long time, that is, random projection is mainly exploited for various tasks of classification, which prefer to  maximize the distances between different classes, rather than  preserve the pairwise distances.  In this sense, we are motivated to study random projection from the viewpoint of feature selection, rather than of traditional distance preservation as required by JL lemma. During this study, however, the property of satisfying JL lemma should not be ignored, because it promises the stability of data structure during random projection, which enables the possibility of conducting classification  in the projection space. Thus throughout the paper, all evaluated  random matrices are previously ensured to satisfy JL lemma to a certain degree.



In this paper, we indeed propose the desired $\{0, \pm1\}$  random projection matrix with the best feature selection performance, by theoretically analyzing the change trend of feature selection performance over the varying sparsity of  random matrices. The proposed matrix presents currently the sparsest structure, which holds only one random nonzero position per column. In theory, it is expected to provide better classification performance over other more dense matrices, if the projection dimension is not much smaller than the number of feature elements. This conjecture is confirmed with extensive classification experiments on both synthetic and real data.

The rest of the paper is organized as follows. In  the next section, the JL lemma is first  introduced, and then the distance preservation property of sparse random matrix over varying sparsity is evaluated. In section \ref{sec-theoryframework}, a theoretical frame is proposed to predict feature selection performance of random matrices over varying sparsity. According to the theoretical conjecture, the currently known sparsest matrix with better performance over other more dense matrices is proposed and analyzed in section \ref{sec-proposal}.  In section \ref{sec-experiment}, the performance advantage of the proposed sparse matrix is verified by performing binary classification on both synthetic data and real data. The real data incudes three representative  datasets in dimension reduction: face image, DNA microarray and text document. Finally, this paper is concluded in section \ref{sec-conclusion}.

\section{Preliminaries}
\label{sec-preliminary}
This section first briefly reviews  JL lemma, and then evaluates the distance preservation of sparse random matrix over varying sparsity.

For easy reading, we begin by introducing some basic notations for this paper. A random matrix is denoted by $\mathbf{R}  \in \mathbb{R}^{k\times d}$, $k<d$.  $r_{ij}$ is used to represent the element of $\mathbf{R}$ at the $i$-th row and the $j$-th column, and $\mathbf{r}\in \mathbb{R}^{1\times d}$ indicates the row vector of $\mathbf{R}$.  Considering the paper is concerned with binary classification, in the following study we tend to  define two samples $\mathbf{v} \in \mathbb{R}^{1\times d}$ and $\mathbf{w}\in \mathbb{R}^{1\times d}$, randomly drawn from  two different patterns of high-dimensional datasets $\mathcal{V} \subset \mathbb{R}^{d}$ and $\mathcal{ W} \subset \mathbb{R}^{d}$, respectively. The inner product between two vectors is typically written as $\langle \mathbf{v},\mathbf{w} \rangle$. To  distinguish from variable, the vector is written in bold. In the proofs of the following lemmas,  we typically use $\Phi(*)$ to denote the cumulative distribution function  of $N(0,1)$.    The minimal integer not less than $*$, and  the the maximum integer not larger than $*$ are  denoted with $\lceil * \rceil$ and $\lfloor * \rfloor$ .


\subsection{Johnson-Lindenstrauss (JL) lemma}
\label{subsec-jl}
\newcommand{\ie}{i.e.}
The distance preservation of random projection is supported by JL lemma. In the past decades, several variants of JL lemma have been proposed in \cite{DasGupta99,Matousek08,Arriaga06}. For the convenience of the proof of the following Corollary 2, here we recall the version of \cite{Arriaga06} in the following Lemma 1. According to Lemma 1, it can be observed  that a random matrix satisfying  JL lemma should have  $\mathds{E}(r_{ij})=0$  and $\mathds{E}(r_{ij}^2)=1$.


\begin{lemma}\cite{Arriaga06}  \quad \label{lemma-1} Consider  random matrix $\mathbf{R}  \in \mathbb{R}^{k\times d}$, with each entry $r_{ij}$ chosen independently from a distribution that is symmetric about the origin with $\mathds{E}(r_{ij}^2)=1$.
For any fixed vector $\mathbf{v}\in \mathbb{R}^d$, let $\mathbf{v}'=\frac{1}{\sqrt{k}}\mathbf{R}\mathbf{v}^{T}$.
\begin{itemize}
\item{} Suppose $B=\mathds{E}(r_{ij}^4)<\infty$. Then for any $\epsilon>0$,

\begin{equation}
\label{eq-prB}
\begin{aligned}
\text{\emph{Pr}}(\|\mathbf{v}'\|^2\leq(1-\epsilon)\|\mathbf{v}\|^2)\leq e^{-\frac{(\epsilon^2-\epsilon^3)k}{2(B+1)}}
\end{aligned}
\end{equation}
\item{} Suppose $\exists L>0$ such that for any integer $m>0$, $\mathds{E}(r_{ij}^{2m})\leq \frac{(2m)!}{2^mm!}L^{2m}$. Then for any $\epsilon >0$,
\begin{equation}
\label{eq-prl}
\begin{aligned}
\text{   \emph{Pr}}(\|\mathbf{v}'\|^2\geq(1+\epsilon)L^2\|\mathbf{v}\|^2)&\leq((1+\epsilon)e^{-\epsilon})^{k/2}\\ &\leq e^{-(\epsilon^2-\epsilon^3)\frac{k}{4}}
\end{aligned}
\end{equation}
\end{itemize}
\end{lemma}

\subsection{Sparse random projection matrices}
\label{subsec-sparse}
Up to now, only a few random matrices are theoretically proposed for random projection. They  can be roughly classified into two typical classes. One is the  Gaussian random matrix with entries i.i.d  dawn from $N(0,1)$ , and the other is the sparse random matrix with elements satisfying the distribution below:

\begin{equation}
\label{eq-rij}
r_{ij}=\sqrt{q}\times
  \left\{
   \begin{array}{cl}
   1&\text{with probability} ~1/2q  \\
  0& \text{with probability} ~1-1/q  \\
  -1&\text{with probability} ~1/2q
   \end{array}
  \right.
  \end{equation}
where $q$ is allowed to be 2, 3 \cite{Achlioptas03} or $\sqrt{d}$ \cite{Li06}. Apparently the larger $q$ indicates the higher sparsity.


Naturally, an interesting question arises:  can we continue improving the sparsity of random projection? Unfortunately, as illustrated in Lemma \ref{lemma-2}, the concentration of JL lemma will decrease as the sparsity increases. In other words, the higher sparsity leads to weaker performance on distance preservation. However, as it will be disclosed in the following part,  the classification tasks involving random projection are more sensitive to  feature selection rather than to distance preservation.

\begin{lemma} \quad  \label{lemma-2} Suppose one class of random matrices  $R  \in \mathbb{R}^{k\times d}$,   with each entry $r_{ij}$ of the distribution as in formula \eqref{eq-rij}, where $q=k/s$ and $1\leq s \leq k$ is an integer. Then these matrices satisfy JL lemma with different levels: the   sparser matrix implies the worse  property on distance preservation.
\end{lemma}
\begin{proof}  With formula \eqref{eq-rij}, it is easy to derive that the proposed matrices satisfy the distribution defined in Lemma \ref{lemma-1}. In this sense, they also obey JL lemma if the two constraints corresponding to formulas \eqref{eq-prB} and \eqref{eq-prl} could be further proved.  \\
 For the first constraint corresponding to formula \eqref{eq-prB}:\\
 \begin{equation}
 \label{eq-B}
 \begin{aligned}
B&=\mathds{E}(r_{ij}^4)\\&=(\sqrt{k/s})^4\times(s/2k)+(-\sqrt{k/s})^4\times(s/2k)\\
&=k/s<\infty
\end{aligned}
\end{equation}
then it is approved.\\
For the second constraint corresponding to formula \eqref{eq-prl}:\\
for any integer $m>0$, derive $\mathds{E}(r^{2m})=(k/s)^{m-1}$, and
$$
\frac{\mathds{E}(r_{ij}^{2m})}{(2m)!L^{2m}/(2^mm!)}=\frac{2^mm!k^{m-1}}{s^{m-1}(2m)!L^{2m}}.
$$
Since $(2m)!\geq m!m^m$, \\
$$
\frac{\mathds{E}(r_{ij}^{2m})}{(2m)!L^{2m}/(2^mm!)}\leq\frac{2^mk^{m-1}}{s^{m-1}m^mL^{2m}},
$$
let $L=(2k/s)^{1/2}\geq \sqrt{2}(k/s)^{(m-1)/2m}/\sqrt{m}$, further derive
$$
\frac{\mathds{E}(r_{ij}^{2m})}{(2m)!L^{2m}/(2^mm!)}\leq 1.
$$
Thus $\exists L=(2k/s)^{1/2}>0$ such that  $$\mathds{E}(r_{ij}^{2m})\leq\frac{(2m)!}{2^mm!}L^{2m} $$ for any integer $m>0$. Then the second constraint  is also proved.\\
Consequently,  it is  deduced that, as $s$ decreases,   $B$ in formula \eqref{eq-B} will increase, and subsequently the boundary error in formula \eqref{eq-B}  will get larger. And this implies that the sparser the matrix is, the  worse the JL property.
\end{proof}

\section{Theoretical Framework}
\label{sec-theoryframework}
  In this section, a theoretical framework is proposed to evaluate the feature selection performance of random matrices with varying sparsity. As it will be shown latter, the feature selection performance  would be simply observed, if  the product between the difference between two distinct high-dimensional vectors and the sampling/row vectors of random matrix, could be easily derived. In this case, we have to previously know  the distribution of the difference between two distinct high-dimensional vectors. For the possibility of analysis, the distribution should be characterized with a unified model. Unfortunately, this goal  seems hard to be perfectly achieved due to the diversity and complexity of natural data. Therefore, without loss of generality, we typically  assume the i.i.d Gaussian distribution for the elements of  difference between two distinct high-dimensional vectors, as detailed in the following section \ref{subsec-difference}. According to the law of large numbers, it can be inferred that the Gaussian distribution is reasonable to be applied to characterize the  distribution of high-dimensional vectors in magnitude. Similarly to most theoretical work attempting to model the real world, our assumption also suffers from an obvious limitation. Empirically, some of the real data elements, in particular the redundant (indiscriminative) elements, tend to be   coherent to some extent, rather than being absolutely independent as we assume above. This imperfection probably limits the accuracy and applicability of our theoretical model. However, as will be detailed later, this problem can be ignored in our analysis where   the difference between pairwise redundant elements is assume to be zero. This also explains why our theoretical proposal can be widely verified in the final  experiments involving a great amount of real data.  With the aforementioned assumption,  in section \ref{subsec-product}, the product between high-dimensional vector difference and row vectors of random matrices is calculated and analyzed with respect to the varying sparsity of random matrix, as detailed in Lemmas 3-5 and related remarks. Note that to make the paper more readable, the proofs of Lemmas 3-5 are included in the Appendices.

\subsection{Distribution of the difference between two distinct high-dimensional vectors}
\label{subsec-difference}
From the viewpoint of feature selection, the random projection is expected to maximize  the difference between arbitrary two samples $\mathbf{v}$ and $\mathbf{w}$ from two different datasets $\mathcal{V}$ and $\mathcal{W}$, respectively. Usually the difference is measured with the Euclidean distance denoted by $\lVert\mathbf{R}\mathbf{z}^T\rVert_2$, $\mathbf{z}=\mathbf{v}-\mathbf{w}$. Then in terms of the mutual independence of $\mathbf{R}$, the search for good random projection is equivalent to seeking the row vector $\hat{\mathbf{r}}$ such that
\begin{equation}
\label{eq-rhat}
 \hat{\mathbf{r}}=\argmax_{\mathbf{r}}\{|\langle \mathbf{r},\mathbf{z} \rangle|\}.
 \end{equation}
Thus in the following part we only need to  evaluate  the   row vectors of $\mathbf{R}$. For the convenience of analysis, the two classes of high-dimensional data are further ideally divided into two parts, $\mathbf{v}=[\mathbf{v}^f ~\mathbf{v}^r]$ and $\mathbf{w}=[\mathbf{w}^f ~\mathbf{w}^r]$, where $\mathbf{v}^f$ and $\mathbf{w}^f$ denote the feature elements containing the discriminative information between $\mathbf{v}$ and $\mathbf{w}$ such that  $\mathds{E}(v^f_i-w^f_i)\neq 0$, while $\mathbf{v}^r$ and $\mathbf{w}^r$ represent the redundant elements such that $\mathds{E}(v^r_i-w^r_i)= 0  $ with a tiny variance. Subsequently, $\mathbf{r}=[\mathbf{r}^f~ \mathbf{r}^r]$ and $\mathbf{z}=[\mathbf{z}^f ~\mathbf{z}^r]$ are also seperated into two parts corresponding to the coordinates of feature elements and redundant elements, respectively. Then the  task of random projection can be reduced to maximizing $|\langle \mathbf{r}^f,\mathbf{z}^f \rangle|$, which implies that the redundant elements have no impact on the feature selection. Therefore, for simpler expression, in the following part the high-dimensional data is assumed to have only feature elements except for specific explanation, and  the superscript $f$ is simply dropped. As for the  intra-class samples, we can simply assume that their elements  are all redundant elements, and then the expected value of their difference is equal to 0, as derived before. This means that the problem of minimizing the  intra-class distance needs not to be  further studied. So in the following part, we only consider the case of maximizing  inter-class distance, as described in formula \eqref{eq-rhat}.



To explore the desired $\hat{\mathbf{r}}_i$ in formula \eqref{eq-rhat}, it is necessary to know the distribution of $\mathbf{z}$. However, in practice the distribution is hard to be characterized since the locations of  feature elements are usually unknown. As a result, we have to make a relaxed assumption on the distribution of $\mathbf{z}$. For a given real dataset, the  values of $v_i$ and $w_i$ should be limited. This  allows us to assume that their difference $z_i$ is also bounded in amplitude, and acts as some unknown distribution. For the sake of generality, in this paper $z_i$ is regarded as approximately satisfying the Gaussian distribution in magnitude and randomly takes a  binary sign.  Then  the distribution of $z_i$ can be formulated as

\begin{equation}
\label{eq-zi}
z_i=\left\{
   \begin{array}{cl}
   x&\text{with probability} ~1/2 \\
  -x&\text{with probability} ~1/2
   \end{array}
  \right.
\end{equation}
where $x\in N(\mu,\sigma^2)$, $\mu$ is a positive number, and Pr$(x>0)=1-\epsilon$, $\epsilon=\Phi(-\frac{\mu}{\sigma})$ is a small positive number.
\subsection{Product between high-dimensional vector and random sampling vector with varying sparsity}
\label{subsec-product}

This subsection mainly tests the feature selection performance of  random row vector with varying sparsity. For the sake of comparison,   Gaussian random vectors are also evaluated. Recall that under the basic requirement of JL lemma, that is $\mathds{E}(r_{ij})=0$ and $\mathds{E}(r_{ij}^2)=1$, the Gaussian matrix has  elements i.i.d drawn from $N(0,1)$, and the sparse random matrix has elements distributed as in formula \eqref{eq-rij} with  $q \in \{d/s:1\leq s\leq d, s\in \mathds{N}\}$.

 Then from the following Lemmas 3-5, we present two crucial random projection results for the high-dimensional data with the feature difference element $|z_i|$ distributed as in formula \eqref{eq-zi}:

\begin{itemize}
\item Random matrices will achieve the best feature selection performance as only one feature element is sampled by each row vector; in other words,  the solution to the formula \eqref{eq-rhat} is obtained when $\mathbf{r}  $ randomly has $s=1$ nonzero elements;
\item The  desired sparse random matrix mentioned above can also obtain better feature selection performance than Gaussian random matrices.
\end{itemize}

Note that, for better understanding, we first prove   a relatively simple case of $z_i\in \{\pm\mu\}$ in Lemma \ref{lemma-3}, and then  in Lemma \ref{lemma-4} generalize to a more complicated case of $z_i$  distributed as  in formula \eqref{eq-zi}. The performance of Gaussian matrices on $z_i\in \{\pm\mu\}$  is obtained in Lemma \ref{lemma-5}.
\begin{lemma} \label{lemma-3}
Let $\mathbf{r}=[r_{1},...,r_{d}]$  randomly have  $1\leq s \leq d$ nonzero elements taking values $\pm\sqrt{d/s}$ with equal probability, and $\mathbf{z}=[z_1,...,z_d]$ with elements being $\pm\mu$ equiprobably, where $\mu$ is a positive constant. Given $f(\mathbf{r},\mathbf{z})=|\langle \mathbf{r},\mathbf{z}\rangle|$, there are three results regarding the expected value of $f(r_i,z)$:

\item [1)]~$ \mathds{E}(f)=2\mu\sqrt{\frac{d}{s}}\frac{1}{2^s}\lceil \frac{s}{2}\rceil C_s^{\lceil \frac{s}{2}\rceil}$;
\item [2)]~$\mathds{E}(f)|_{s=1}=\mu\sqrt{d}>\mathds{E}(f)|_{s>1}$;
\item[3)]~$\mathop{\lim}\limits_{s\rightarrow \infty} \frac{1}{\sqrt{d}}\mathds{E}(f)\rightarrow \mu\sqrt{\frac{2}{\pi}}$.
\end{lemma}
\begin{proof}
Please see Appendix A.
\end{proof}
\noindent {\bf Remark on Lemma \ref{lemma-3}:} This lemma discloses that the best feature selection performance is obtained, when only one feature element is sampled by each row vector. In contrast, the performance tends to converge to a lower level as the number of sampled feature elements increases. However, in practice the desired sampling process is hard to be implemented due to the few knowledge of  feature location. As it will be detailed in the next section, what we can really implement is  to  sample only one feature element with high probability. Note that with the proof of this lemma, it can also be proved  that if $s$ is odd, $\mathds{E}(f)$  fast decreases to $\mu\sqrt{2d/\pi}$ with increasing $s$; in contrast, if $s$ is even, $\mathds{E}(f)$ quickly increases towards $\mu\sqrt{2d/\pi}$ as $s$ increases. But for arbitrary two adjacent $s$  larger than 1, their average value on $\mathds{E}(f)$, namely $(\mathds{E}(f)|_{s}+\mathds{E}(f)|_{s+1})/2$, is very close to $\mu\sqrt{2d/\pi}$. For clarity, the values of $\mathds{E}(f)$ over varying $s$ are calculated and shown in Figure \ref{fig-lemma3}, where instead of $\mathds{E}(f)$, $\frac{1}{\mu\sqrt{d}}\mathds{E}(f)$ is described since only the varying  $s$ is concerned. The specific character of $\mathds{E}(f)$  ensures that one can still achieve better performance over others by  sampling $s=1$ element with a relative high probability, along with the occurrence of a sequence of $s$  taking consecutive values slightly larger than 1.

\begin{figure*}[t]
\renewcommand{\captionfont}{\small}
\centering
\begin{tabular}{cc}
\includegraphics[width=0.45\textwidth]{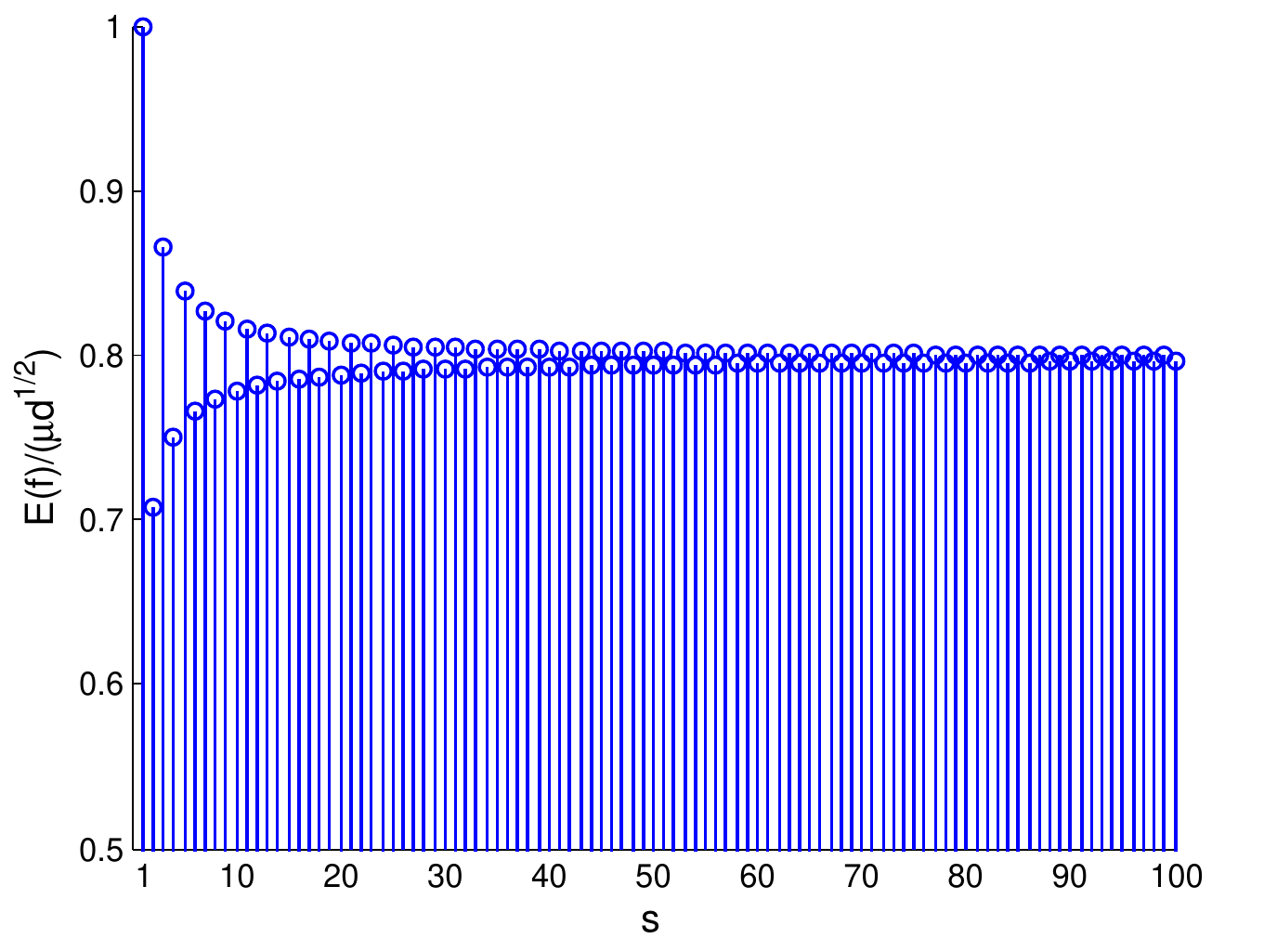}&\includegraphics[width=0.45\textwidth]{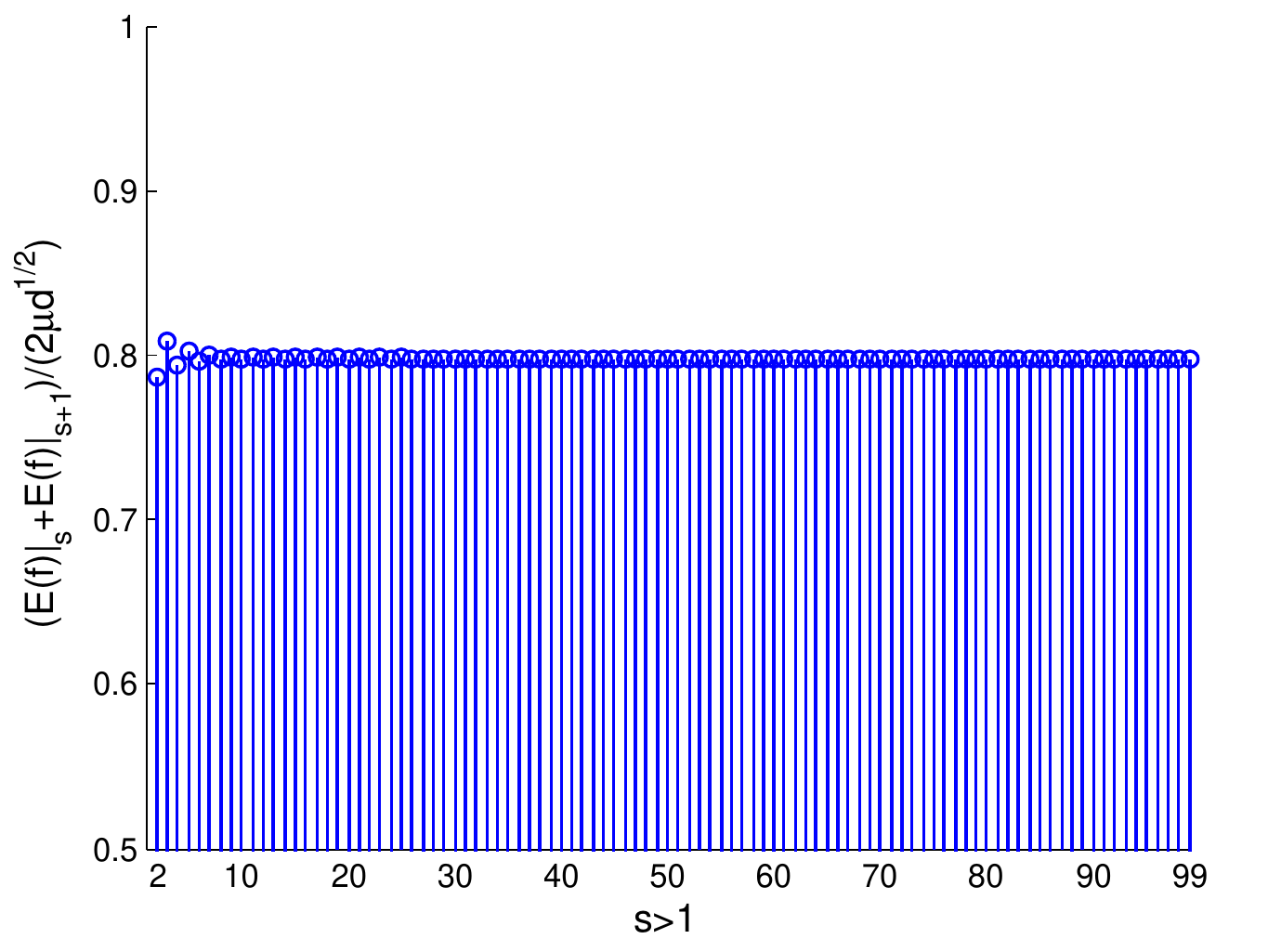}\\
(a)&(b)
\end{tabular}

\caption{ The process of $\frac{1}{\mu\sqrt{d}}\mathds{E}(f)$ converging to $\sqrt{2/\pi}~(\approx 0.7979)$ with increasing $s$ is described in (a); and in (b) the average value of two $\frac{1}{\mu\sqrt{d}}\mathds{E}(f)$  with adjacent  $s~(>1)$, namely $\frac{1}{2\mu\sqrt{d}}(\mathds{E}(f)|_s+\mathds{E}(f)|_{s+1})$,  is approved very close to $\sqrt{2/\pi}$. Note that $\mathds{E}(f)$ is calculated with the formula provided in Lemma \ref{lemma-3}. \label{fig-lemma3}}
\end{figure*}
\begin{lemma} \label{lemma-4}
Let $\mathbf{r}=[r_{1},...,r_{d}]$  randomly have  $1\leq s \leq d$ nonzero elements taking values $\pm\sqrt{d/s}$ with equal probability, and $\mathbf{z}=[z_1,...,z_d]$ with elements distributed as in formula \eqref{eq-zi}. Given $f(\mathbf{r},\mathbf{z})=|\langle \mathbf{r},\mathbf{z}\rangle|$, it is derived that:

$$\mathds{E}(f)|_{s=1}>\mathds{E}(f)|_{s>1}$$  if~  $(\frac{9}{8})^{\frac{3}{2}}[\sqrt{\frac{2}{\pi}}+(1+\frac{\sqrt{3}}{4})\frac{2}{\pi}(\frac{\mu}{\sigma})^{-1}]+2\Phi(-\frac{\mu}{\sigma})\leq 1$.
\end{lemma}
\begin{proof}
Please see Appendix B.
\end{proof}
\noindent {\bf Remark on Lemma \ref{lemma-4}:} This lemma expands  Lemma \ref{lemma-3} to a more general case where  $|z_i|$ is allowed to vary in some range. In other words, there is an  upper bound on $\frac{\sigma}{\mu}$ for $\mathds{E}(f)|_{s=1}>\mathds{E}(f)|_{s>1}$, since $\Phi(-\frac{\mu}{\sigma})$ decreases monotonically with respect to $\frac{\mu}{\sigma}$. Clearly the larger upper bound  for $\frac{\sigma}{\mu}$ allows more variation of $|z_i|$.  In practice the real upper bound  should be larger than that we have derived as a sufficient condition in this   lemma.

\vspace{6pt}
\begin{lemma} \label{lemma-5}
Let $\mathbf{r}=[r_{1},...,r_{d}]$  have elements i.i.d drawn from $N(0,1)$, and $\mathbf{z}=[z_1,...,z_d]$ with elements being $\pm\mu$ equiprobably, where $\mu$ is a positive constant. Given $f(\mathbf{r},\mathbf{z})=|\langle \mathbf{r}, \mathbf{z}\rangle|$, its expected value  $\mathds{E}(f)=\mu\sqrt{\frac{2 d}{\pi}}$.
\end{lemma}
\begin{proof}
Please see Appendix C.
\end{proof}
\noindent {\bf Remark on Lemma \ref{lemma-5}:} Comparing this lemma with Lemma \ref{lemma-3},  clearly the row vector with Gaussian distribution shares the same feature selection level with sparse row vector with a relatively large $s$. This explains why in practice the sparse random matrices usually can present comparable classification performance  with Gaussian matrix. More importantly, it implies that the sparsest sampling process provided in Lemma 3  should outperform Gaussian matrix on feature selection.
\section{ Proposed sparse random matrix}
\label{sec-proposal}
The lemmas of the former section have proved that the best feature selection performance can be obtained, if only one feature element is sampled by each row vector of random matrix.
It is now interesting to know if the condition above can be satisfied in the practical setting,  where the high-dimensional data consists of both feature elements and redundant elements, namely $\mathbf{v}=[\mathbf{v}^f ~\mathbf{v}^r]$ and $\mathbf{w}=[\mathbf{w}^f ~\mathbf{w}^r]$. According to the theoretical condition mentioned above, it is known  that the row vector  $\mathbf{r}=[\mathbf{r}^f ~ \mathbf{r}^r]$ can obtain the best feature selection, only when  $||\mathbf{r}^f||_0=1$, where the quasi-norm $\ell_0$  counts the number of nonzero elements in $\mathbf{r}^f$.  Let $\mathbf{r}^f\in\mathds{R}^{d_f}$, and $\mathbf{r}^r\in\mathds{R}^{d_r}$, where $d=d_f+d_r$. Then the desired row vector should have $d/d_f$  uniformly distributed nonzero elements such that $\mathds{E}(||\mathbf{r}^f||_0)=1$.  However, in practice the desired distribution for row vectors is often hard to be determined, since for a real dataset   the number of feature elements is usually unknown.

In this sense, we are motivated to propose a general  distribution for the matrix elements, such that $||\mathbf{r}^f||_0=1$ holds with high probability  in the setting where the feature distribution is unknown. In other words, the random matrix should hold the distribution maximizing the ratio $\text{Pr}(||\mathbf{r}^f||_0=1)/\text{Pr}(||\mathbf{r}^f||_0\in\{2,3,...,d_f\})$. In practice, the desired distribution implies that the random matrix has exactly one nonzero position  per column, which can be simply derived as below. Assume  a random matrix $\mathbf{R} \in \mathbb{R}^{k\times d}  $  randomly holding $1\leq s'\leq k$ nonzero elements  per \emph{column},  equivalently $s'd/k$ nonzero elements per \emph{row}, then  one can derive that


\begin{equation}
\label{eq-pr}
 \begin{aligned}
&\text{Pr}(||\mathbf{r}^f||_0=1)/\text{Pr}(||\mathbf{r}^f||_0\in\{2,3,...,d_f\})\\&=\frac{\text{Pr}(||\mathbf{r}^f||_0=1)}{1-\text{Pr}(||\mathbf{r}^f||_0=0)-\text{Pr}(||\mathbf{r}^f||_0=1)}\\
&=\frac{C_{d_f}^1C_{d_r}^{s'd/k-1}}{C_d^{s'd/k}-C_{d_r}^{s'd/k}-C_{d_f}^1C_{d_r}^{s'd/k-1}}\\
&=\frac{d_fd_r!}{\frac{d!(d_r-s'd/k+1)!}{s'd/k(d-s'd/k)!}-\frac{d_r!(d_r-s'd/k+1)}{s'd/k}-d_fd_r!}
 \end{aligned}
\end{equation}
 From  the last equation in formula \eqref{eq-pr}, it can be observed that the increasing $s'd/k$ will reduce the value of formula \eqref{eq-pr}. In order to maximize the value,  we have to set $s'=1$. This indicates that the desired random matrix has only one nonzero  element per column.

The proposed random matrix with exactly one nonzero element per column presents two obvious advantages, as detailed below.

\begin{itemize}
\item In complexity, the proposed matrix clearly presents much higher sparsity than existing random projection matrices.  Note that, theoretically the very sparse random matrix with $q=\sqrt{d}$ \cite{Li06} has higher sparsity than the proposed matrix when $k<\sqrt{d}$. However, in practice the case $k<\sqrt{d}$ is usually not of practical interest,  due to the weak performance caused by  large compression rate $d/k$ ($>\sqrt{d}$).
    \item In performance, it can be derived that the proposed matrix outperforms other more dense matrices, if the projection dimension $k$ is not much smaller than the number $d_f$ of feature elements included in the high-dimensional vector. To be specific,  from Figure \ref{fig-lemma3}, it can be observed  that the dense matrices with column weight $s'>1$ share comparable  feature selection performance, because as $s'$ increases they tend to sample more than one feature element (namely $||\mathbf{r}^f||_0>1$) with higher probability. Then the proposed  matrix with $s'=1$ will present better performance  than them, if $k$ ensures $||\mathbf{r}^f||_0=1$ with high probability, or equivalently the ratio $\text{Pr}(||\mathbf{r}^f||_0=1)/ \text{Pr}(||\mathbf{r}^f||_0\in\{2,3,...,d_f\})$ being relatively large. As shown in formula \eqref{eq-pr}, the condition above can be  better satisfied, as $k$ increases. Inversely, as $k$ decreases, the feature selection advantage of the proposed matrix will degrade. Recall that the proposed matrix is weaker than other more dense matrices on distance preservation, as demonstrated in section \ref{subsec-sparse}. This means that the proposed matrix will perform worse than others when its feature selection advantage is not obvious. In other words, there should exist  a  lower bound for $k$ to ensure the performance advantage of the proposed matrix, which is also verified in the following experiments.  It can be roughly estimated that the  lower bound of $k$ should be  on the order of $d_f$, since for the proposed matrix with column weight $s'=1$, the $k=d_f$ leads to $\mathds{E}(||\mathbf{r}^f||_0)=d/k\times d_f/d=1$. In practice,   the performance advantage seemingly can be maintained  for a relatively small $k(<d_f)$. For instance, in the following  experiments on synthetic data,  the lower bound of $k$ is as small as $d_f/20$. This phenomenon  can be explained by the fact that to obtain performance advantage, the probability  $\text{Pr}(||\mathbf{r}^f||_0=1)$ is only required to be relatively large rather than to be equal to 1, as demonstrated in the remark on Lemma \ref{lemma-3}.

\end{itemize}

\section{Experiments }
\label{sec-experiment}
\subsection{Setup}
This section verifies the feature selection advantage of the proposed currently sparest matrix (StM) over other popular matrices, by  conducting binary classification on both synthetic data and real data. Here the synthetic data with labeled feature elements is provided to specially observe the relation between the projection dimension and feature number, as well as the impact of redundant elements.  The real data involves three typical  datasets in the area of dimensionality reduction: face image, DNA microarray and text document. As for the binary classifier, the classical support vector machine (SVM) based on Euclidean distance is adopted.   For comparison, we  test three popular random matrices: Gaussian random matrix (GM), sparse random matrix (SM)  as in formula \eqref{eq-rij} with $q=3$ \cite{Achlioptas03} and very sparse random matrix (VSM) with $q=\sqrt{d}$ \cite{Li06}.


The simulation parameters are introduced as follows. It is known that the repeated random projection tends to improve the feature selection, so here each classification decision is voted by performing 5 times the random projection \cite{Fern03}. The  classification accuracy at each projection dimension $k$ is derived by taking the \textbf{average} of  \textbf{100000} simulation runs. In each simulation, four matrices are tested with the same samples. The projection dimension $k$ decreases uniformly from the high dimension $d$. Moreover, it is necessary to note that, for some  datasets containing more than two classes of samples,  the SVM classifier randomly selects two classes to conduct binary classification in each simulation. For each class of data, one half of samples are randomly selected for training, and the rest for testing.
\begin{table*}[htp]

\caption{ Classification accuracies on the synthetic data which have $d=2000$ and redundant elements suffering from three different varying levels $\sigma_r$. The best performance is highlighted in bold. The lower bound of projection dimension $k$ that ensures the proposal outperforming others in all datasets is  highlighted in bold as well. Recall that the acronyms GM, SM, VSM and StM represent Gaussian random matrix, sparse random matrix with $q=3$, very sparse rand matrix with $q=\sqrt{d}$, and the proposed sparsest random matrix, respectively. \label{tab-synthetic}}

\centering

\small{
\begin{tabular}{|c||r|c|c|c|c|c|c|c|c|c|}
\hline

&$k$&50 &100& \textbf{200}& 400& 600 &800& 1000& 1500& 2000\\\hline\hline

\multirow{4}{*}{\rotatebox{90}{$\sigma_r=8$}}&GM& 70.44& 67.93& 84.23& 93.31& 95.93& 97.17& 97.71& 98.35& 98.74\\
&SM& \textbf{70.65}&  67.90& 84.43& 93.03& 95.97& 96.86& 97.78& 98.36&  98.80\\
&VSM& 70.55& 68.05& 84.46& 93.19&  96.00& 96.99& 97.68& 98.38& 98.76\\
&StM& 70.27& \textbf{68.09}& \textbf{84.66}& \textbf{94.22}& \textbf{97.11}& \textbf{98.03}&\textbf{98.67}& \textbf{99.37}& \textbf{99.57}\\
\hline\hline
\multirow{4}{*}{\rotatebox{90}{$\sigma_r=12$}}&GM &\textbf{ 64.89}& \textbf{63.06}& 76.08& 85.04&88.46& 90.21& 91.16& 92.68&93.32\\
&SM& 64.67& 62.66& 75.85& 85.03&  88.30&90.09& 91.21&  92.70&  93.30\\
&VSM& 65.17& 62.95& 76.12& 85.14&  88.80& 90.46& 91.37& 92.88&93.64\\
&StM& 64.85& 63.00&\textbf{76.82}&\textbf{88.41}& \textbf{93.51}& \textbf{96.12}& \textbf{97.59}&\textbf{99.13}&\textbf{ 99.68}\\
\hline\hline

\multirow{4}{*}{\rotatebox{90}{$\sigma_r=16$}}&GM &60.90& 59.42& 70.13& 78.26& 81.70&83.82& 84.74&  86.50& 87.49\\
&SM& 60.86&59.58&69.93& 78.04&81.66& 83.85&84.79& 86.55&87.39\\
&VSM& 60.98& \textbf{59.87}& 70.27& 78.49& 81.98& 84.36& 85.27& 86.98& 87.81\\
&StM& \textbf{61.09}& 59.29& \textbf{71.58}& \textbf{84.56}& \textbf{91.65}&  \textbf{95.50}& \textbf{97.24}&\textbf{98.91}& \textbf{99.30}\\
\hline
\end{tabular}
}
\end{table*}
\subsection {Synthetic data experiments}

\subsubsection{Data generation}

The synthetic data is developed to evaluate the two factors as follows:
\begin{itemize}
\item the relation between the lower bound of projection dimension $k$ and the feature dimension $d_f$;
\item the negative impact of  redundant elements, which are ideally assumed to be zero in the previous theoretical proofs.
\end{itemize}
To this end, two classes of synthetic data with $d_f$ feature elements and $d-d_f$ redundant elements are generated in two  steps:
\begin{itemize}
\item  randomly build a vector $\tilde{\mathbf{v}}\in\{\pm1\}^d$, then define a vector $\tilde{\mathbf{w}}$ distributed as $\tilde{w}_i=-\tilde{v}_i$, if $1\leq i\leq d_f$, and $\tilde{w}_i=\tilde{v}_i$, if $d_f< i\leq d$;
\item  generate two classes of datasets $\mathcal{V}$ and $\mathcal{W}$  by i.i.d sampling $v^f_i\in N(\tilde{v}_i,\sigma_f^2)$ and $w^f_i\in N(\tilde{w}_i,\sigma_f^2)$, if $1\leq i\leq d_f$; and $v^r_i\in N(\tilde{v}_i,\sigma_r^2)$ and $w^r_i\in N(\tilde{w}_i,\sigma_r^2)$, if $d_f< i\leq d$.
\end{itemize} Subsequently, the distributions on pointwise distance can be approximately derived as  $|v_i^f-w_i^f|\in N(2,2\sigma_f^2)$ for feature elements and  $(v_i^r-w_i^r)\in N(0,2\sigma_r^2)$ for redundant elements, respectively. To be close to reality, we introduce some unreliability for   feature elements and redundant elements by adopting relatively large variances. Precisely, in the simulation $\sigma_f$ is fixed to 8 and $\sigma_r$ varies in the set $\{8, 12, 16\}$. Note that, the probability of $(v_i^r-w_i^r)$ converging to zero will decrease as $\sigma_r$ increases. Thus the increasing $\sigma_r$ will be a challenge for our previous theoretical conjecture derived on the  assumption of  $(v_i^r-w_i^r)=0$.  As for the size of the dataset, the data dimension $d$ is set to 2000, and the feature dimension $d_f=1000$. Each dataset consists of  100 randomly generated samples.
\subsubsection{Results}
 Table \ref{tab-synthetic} shows the  classification performance of four types of matrices over evenly varying projection dimension $k$. It is clear that the proposal  always outperforms others, as $k>200$ (equivalently, the compression ratio $k/d>0.1$). This result exposes two positive clues. First,  the proposed  matrix preserves  obvious advantage over others, even when  $k$ is relatively small, for instance, $k/d_f$ is allowed to be as small as 1/20 when $\sigma_r=8$. Second, with the interference of redundant elements, the proposed matrix still outperforms others, which implies that the previous theoretical result is also applicable to the real case where the redundant elements cannot be simply neglected. 


\subsection{Real data experiments}

Three types of representative high-dimensional datasets are tested for random projection over evenly varying projection dimension $k$. The datasets are first briefly introduced, and then the results are illustrated and analyzed. Note that, the simulation is developed to compare the feature selection performance of different random projections, rather than to obtain the best performance. So to reduce the simulation load,  the original high-dimensional data is uniformly downsampled to a relatively low dimension. Precisely, the  face image, DNA, and text are reduced to the dimensions 1200, 2000 and 3000, respectively. Note that, in terms of JL lemma, the original high dimension allows to be reduced to  arbitrary values (not limited to 1200, 2000 or 3000), since theoretically the distance preservation of random projection is independent of the size of high-dimensional data \cite{Achlioptas03}.


\subsubsection {Datasets}
\begin{itemize}
        \item   Face image\begin{itemize}
                            \item AR  \cite{Martinez98} : As in \cite{Martinez01}, a subset of 2600 frontal faces from 50 males and 50 females are examined. For some persons, the faces were taken at different times, varying the lighting, facial expressions (open/closed eyes, smiling/not smiling) and facial details (glasses/no glasses). There are 6 faces with dark glasses and 6 faces partially disguised by scarfs among 26 faces per person.
                            \item    Extended Yale B  \cite{Georghiades01,Lee05}: This dataset includes about 2414 frontal faces of 38 persons, which suffer varying illumination changes.
                            \item FERET  \cite{Phillips98}: This dataset consists of more than 10000 faces from more than 1000 persons taken in largely varying circumstances. The database is further divided into several sets which are formed for different evaluations. Here we evaluate the 1984 \emph{frontal} faces of 992 persons each with 2 faces separately extracted from sets \emph{fa} and \emph{fb}.
                            \item  GTF  \cite{Nefian00}:  In this dataset, 750 images from 50 persons were captured at different scales and orientations under variations in illumination and expression. So the cropped faces suffer from serious pose variation.
                            \item ORL  \cite{Samaria94}: It contains 40 persons each with 10 faces. Besides slightly varying lighting and expressions, the faces also undergo slight changes on pose.

                          \end{itemize}

                          \begin{table*}[t]

\caption{  Classification accuracies on five face datasets with  dimension $d= 1200$. For each projection dimension $k$, the best performance is highlighted in bold. The lower bound of projection dimension $k$ that ensures the proposal outperforming others in all datasets is  highlighted in bold as well. Recall that the acronyms GM, SM, VSM and StM represent Gaussian random matrix, sparse random matrix with $q=3$, very sparse random matrix with $q=\sqrt{d}$, and the proposed sparsest random matrix, respectively. \label{tab-face}}
\begin{center}
{\small
\begin{tabular}{|c||r|c|c|c|c|c|c|c|}
\hline
&$k$&30&60&\textbf{120}&240&360&480&600\\\hline\hline

 \multirow{4}{*}{\rotatebox{90}{AR}}&GM &\textbf{98.67}  &  99.04  &  99.19  &  99.24  &  99.30 &   99.28 &   99.33  \\
 &SM &98.58 &   99.04  &  99.21 &   99.25  &  99.31  &  99.30  &  99.32 \\
& VSM & 98.62  &  99.07  &  99.20  &  99.27   & 99.30 &   99.31  &  99.34 \\
& StM& 98.64 &   \textbf{99.10 } &  \textbf{99.24 } &  \textbf{99.35}  &  \textbf{99.48} &   \textbf{ 99.50}  &  \textbf{99.58} \\
 \hline\hline

 \multirow{4}{*}{\rotatebox{90}{Ext-YaleB}}&GM &97.10  &  \textbf{98.06} &   98.39  &  98.49  &  98.48  &  98.45 &   98.47 \\
 &SM &97.00 &  98.05  &  98.37 &   98.49   & 98.48  &  98.45   & 98.47 \\
& VSM &97.12  &  98.05   & 98.36  &  98.50  &  98.48  &  98.45   & 98.48 \\
& StM &\textbf{97.15}  &  \textbf{98.06}  &  \textbf{98.40 } &  \textbf{98.54}  &  \textbf{98.54}   & \textbf{98.57 }&   \textbf{98.59}\\
 \hline\hline

 \multirow{4}{*}{\rotatebox{90}{FERET}}&GM &86.06  &  86.42   & 86.31   & 86.50  &  86.46 &   86.66  &  86.57\\
 &SM &86.51  &  86.66  &  87.26 &   88.01  &  88.57 &   89.59  &  90.13 \\
& VSM&\textbf{87.21}  &  87.61  &  89.34  &  91.14   & 92.31 &   93.75  &  93.81 \\
& StM &87.11 &   \textbf{88.74}  &  \textbf{92.04}  & \textbf{ 95.38}  & \textbf{ 96.90}  &  \textbf{97.47}  &  \textbf{97.47} \\
 \hline\hline

 \multirow{4}{*}{\rotatebox{90}{GTF}}&GM&96.67  &  97.48  &  97.84 &   98.06  &  98.09  &  98.10  &  98.16\\
 &SM &96.63  &  97.52  &  97.85  &  98.06 &   98.09  &  98.13  &  98.16 \\
& VSM &\textbf{96.69 } &  \textbf{97.57} &  97.87  &  98.10  &  98.13  &  98.14   & 98.16 \\
& StM &96.65  &  97.51 &   \textbf{97.94} &   \textbf{98.25} &   \textbf{98.40} &   \textbf{98.43}  &  \textbf{98.53} \\
 \hline\hline

 \multirow{4}{*}{\rotatebox{90}{ORL}}&GM &94.58  &  95.69  &  96.31 &   96.40 &   96.54   & 96.51  &  96.49\\
 &SM &94.50  &  95.63  &  96.36  &  96.38  &  96.48  &  96.47  &  96.48\\
& VSM &94.60  &  \textbf{95.77 } &  96.33  &  96.35  &  96.53 &   96.55 &   96.46 \\
& StM &  \textbf{94.64} &   95.75  &  \textbf{96.43} &   \textbf{96.68}  &  \textbf{96.90}&    \textbf{97.04}  & \textbf{ 97.05} \\
 \hline

      \end{tabular}
      }
\end{center}
\end{table*}

        \item  DNA microarray
                         \begin{itemize}
                         \item  Colon \cite{Alon08061999}: This is a dataset consisting of  40 colon tumors and 22 normal colon tissue samples. 2000 genes with highest intensity across the samples are considered.
                         \item  ALML \cite{Golub15101999}: This dataset contains 25 samples taken from patients suffering from acute myeloid
                           leukemia (AML) and  47 samples from patients suffering from acute lymphoblastic leukemia (ALL). Each sample is expressed with 7129 genes.

                         \item  Lung \cite{Beer081012} : This dataset contains 86 lung tumor and 10 normal lung samples. Each sample holds 7129 genes.

                          \end{itemize}
        \item  Text document \cite{Caideng09}\footnote{Publicly available at \url{http://www.cad.zju.edu.cn/home/dengcai/Data/TextData.html}}
        \begin{itemize}
          \item TDT2: The recently modified dataset includes 96 categories of total 10212 documents/samples. Each document is represented with vector of length 36771.  This paper adopts the first 19 categories each with more than 100 documents, such that each category is tested with  100 randomly selected documents.
          \item 20Newsgroups (version 1): There are 20 categories of 18774 documents in this dataset. Each document has vector dimension 61188. Since the documents are not equally distributed in the 20 categories, we randomly select  600 documents for each category, which is nearly the maximum number we can assign to all categories.
          \item RCV1: The original dataset contains 9625 documents each with 29992 distinct words, corresponding to 4 categories  with 2022, 2064, 2901, and 2638 documents respectively. To reduce computation, this paper randomly selects only 1000 documents for each category.
        \end{itemize}
      \end{itemize}

\subsubsection{Results}

\begin{table*}[t]

\caption{ Classification accuracies on three DNA datasets with dimension $d= 2000$. For each projection dimension $k$, the best performance is highlighted in bold. The lower bound of projection dimension $k$ that ensures the proposal outperforming others in all datasets is  highlighted in bold as well. Recall that the acronyms GM, SM, VSM and StM represent Gaussian random matrix, sparse random matrix with $q=3$, very sparse random matrix with $q=\sqrt{d}$, and the proposed sparsest random matrix, respectively. \label{tab-dna}}

\centering
{\small
\begin{tabular}{|c||r|c|c|c|c|c|c|c|c|}
\hline
&$k$&50&\textbf{100}&200&400&600&800&1000&1500\\\hline\hline

 \multirow{4}{*}{\rotatebox{90}{Colon}} &GM  &77.16&  77.15 &   77.29  &  77.28  &  77.46  &  77.40 &  77.35  &  77.55\\
&SM &\textbf{77.23} & 77.18   & 77.16   & 77.36   & 77.42   & 77.42  &  77.39   & 77.54\\
&VSM & 76.86 & 77.19   & 77.34  &  77.52  &  77.64  &  77.61  & 77.61   & 77.82\\
&StM &76.93  &\textbf{77.34}  &  \textbf{77.73}   & \textbf{78.22}   & \textbf{78.51}   & \textbf{78.67}   & \textbf{78.65}   & \textbf{78.84}\\
\hline\hline

 \multirow{4}{*}{\rotatebox{90}{ALML}}   & GM &\textbf{65.11} &66.22 &   66.96  &  67.21 &   67.23  &  67.24  &  67.28  &  67.37\\
 &SM &65.09  &66.16  &  66.93  &  67.25  &  67.22   & 67.31   & 67.31   & 67.36\\
 &VSM&64.93 & 67.32&68.52&69.01&69.15&69.16&69.25&69.33\\
 &StM &65.07 &\textbf{68.38}&\textbf{70.43}&\textbf{71.39}&\textbf{71.75}&\textbf{71.87}&\textbf{72.00}&\textbf{72.11}\\
 \hline\hline

 \multirow{4}{*}{\rotatebox{90}{Lung}}& GM  &98.74  &  98.80&98.91&98.96&98.95&98.96&98.95&98.97\\
&SM & 98.71  & 98.80&98.92&98.97&98.96&98.98&98.97&98.97\\
&VSM&  \textbf{98.81}  &  99.21 &99.48 &99.57 &99.58 &99.61 &99.61 &99.61\\
&StM & 98.70 &\textbf{99.48}&\textbf{99.69}&\textbf{99.70}&\textbf{99.69}&\textbf{99.72}&\textbf{99.68}&\textbf{99.65}\\
\hline
 \end{tabular}
 }
\end{table*}

\begin{table*}[ht!]
\caption{ Classification accuracies on three Text datasets with dimension $d= 3000$. For each projection dimension $k$, the best performance is highlighted in bold. The lower bound of projection dimension $k$ that ensures the proposal outperforming others in all datasets is  highlighted in bold as well. Recall that the acronyms GM, SM, VSM and StM represent Gaussian random matrix, sparse random matrix with $q=3$, very sparse random matrix with $q=\sqrt{d}$, and the proposed  sparsest random matrix, respectively. \label{tab-text}}

\centering
{\small
\begin{tabular}{|c||r|c|c|c|c|c|c|c|}
\hline
&k&150 &300 &\textbf{600} &900& 1200 &1500& 2000\\
\hline\hline
 \multirow{4}{*}{\rotatebox{90}{TDT2}}& GM&  \textbf{83.64}&   83.10&  82.84&  82.29&  81.94&  81.67&  81.72\\
& SM& 83.61&  82.93&  83.10&  82.28&  81.92&  81.55&  81.76\\
& VSM&  82.59&  82.55&  82.72&   82.20&  81.74&  81.47&  81.78\\
& StM&  82.52&  \textbf{83.15}& \textbf{ 84.06}& \textbf{ 83.58}&  \textbf{83.42}&  \textbf{82.95}&  \textbf{83.35}\\
\hline\hline
 \multirow{4}{*}{\rotatebox{90}{Newsgroup}}
&GM  &  \textbf{ 75.35}&   \textbf{74.46}&   72.27&   71.52&   71.34&   70.63&   69.95\\
& SM &   75.21&   74.43&   72.29&    71.30&   71.07&   70.34&   69.58\\
& VSM &   74.84&   73.47&   70.22&   69.21&   69.28&   68.28&   68.04\\
&StM  &   74.94&    74.20&   \textbf{72.34}&   \textbf{71.54}&   \textbf{71.53}&   \textbf{70.46}&    \textbf{70.00}\\
\hline\hline
\multirow{4}{*}{\rotatebox{90}{RCV1}}
&GM& 85.85& 86.20& 81.65& 78.98& 78.22& 78.21& 78.21\\
&SM& \textbf{86.05}& 86.19& 81.53& 79.08& 78.23& 78.14&78.19\\
&VSM& 86.04& 86.14& 81.54& 78.57& 78.12& 78.05& 78.04\\
&StM&85.75& \textbf{86.33}&\textbf{85.09}& \textbf{83.38}& \textbf{ 82.30}& \textbf{81.39}& \textbf{80.69}\\
\hline

\end{tabular}
}
\end{table*}
Tables \ref{tab-face}-\ref{tab-text} illustrate the classification performance of four classes of matrices on three typical high-dimensional data: face image, DNA microarray and text document. It can be observed that, all results are  consistent with the theoretical conjecture stated in section \ref{sec-proposal}. Precisely,   the proposed matrix will always perform better than others, if $k$ is larger than some thresholds, i.e. $k>120$ (equivalently, the compression ratio $k/d>1/10$) for all face image data, $k>100$ ($k/d>1/20$) for all  DNA data,  and $k>600$ ($k/d>1/5$) for all text data. Note that, for some individual datasets, in fact we can obtain smaller thresholds than the uniform thresholds described above, which means that for these datasets, our performance advantage can be ensured in lower projection dimension. It is worth noting that our performance gain usually varies across the types of data. For most data, the gain is on the level of around $1\%$, except for some special cases, for which the gain can achieve as large as around $5\%$. Moreover, it should be noted that the proposed matrix can still present comparable performance with others (usually inferior to the best results not more than $1\%$), even as $k$ is smaller than the lower threshold described above. This implies that regardless of the value of $k$, the proposed matrix is always valuable due to its  lower  complexity and competitive performance.  In short, the extensive experiments on real data sufficiently verifies the performance advantage of the theoretically proposed random  matrix, as well as the conjecture that the performance advantage holds only when the projection dimension $k$ is large enough. 



\section{Conclusion and Discussion}
\label{sec-conclusion}

 This paper has proved that random projection can achieve its  best feature selection performance,   when only one feature element of high-dimensional data is considered at each sampling. In practice,  however, the number of  feature elements is usually unknown, and so the aforementioned best sampling process is hard to be implemented. Based on the principle of achieving the best sampling process with high probability, we practically propose a class of sparse random matrices with exactly one nonzero element per column, which is expected to outperform other more dense random projection matrices, if the projection dimension is not much smaller than the number of feature elements. Recall that for the possibility of theoretical analysis, we have typically assumed that the elements of high-dimensional data are mutually independent, which obviously cannot be well satisfied by the real data, especially the redundant elements. Although  the impact of  redundant elements is reasonably avoided in our analysis, we cannot ensure that all analyzed  feature elements are exactly independent in practice. This defect might affect the applicability of our theoretical proposal to some extent, whereas empirically the negative impact seems to be negligible, as proved by the experiments on synthetic data. In order to validate the feasibility of the theoretical proposal,  extensive  classification experiments are conducted on various real data, including face image, DNA microarray and text document. As it is expected, the proposed random matrix  shows better performance than other more dense matrices,  as the projection dimension is sufficiently large; otherwise, it  presents comparable performance with others. This result  suggests that for random projection applied to the task of classification, the proposed currently sparsest random matrix is much more attractive  than  other more dense random  matrices in terms of   both complexity and performance.

\section*{Appendix A.}
\label{apd-a}
\noindent{\bf Proof of Lemma \ref{lemma-3}}
\begin{proof}
\noindent
Due to the sparsity of $\mathbf{r}$ and the symmetric property of both $r_{j}$ and $z_j$, the function $f(\mathbf{r},\mathbf{z})$ can be equivalently transformed to a  simpler form, that is $f(x)=\mu\sqrt{\frac{d}{s}}|\sum_{i=1}^{i=s}x_i|$ with $x_i$ being $\pm1$ equiprobably. With the simplified form, three  results of this lemma are sequentially proved below.
\begin{itemize}
  \item[1)] First, it can be easily derived that
$$
  \mathds{E}(f(x))=\mu\sqrt{\frac{d}{s}}\frac{1}{2^s}\sum_{i=1}^{s}(C_s^i|s-2i|)
$$ then the solution to $\mathds{E}(f(x))$  turns to calculating  $\sum_{i=0}^{s}(C_s^i|s-2i|)$, which can be deduced as

$$
\sum_{i=0}^{s}(C_s^i|s-2i|)=\left\{
\begin{array}{cl}
2sC_{s-1}^{\frac{s}{2}-1}& if~ s ~is ~even\\[5pt]
2sC_{s-1}^{\frac{s-1}{2}}& if~ s ~is~ odd\\
\end{array}
\right.
$$ by summing the piecewise function

$$
C_s^i|s-2i|=\left\{
\begin{array}{ll}
sC_{s-1}^{0}& if~ i=0 \\[6pt]
sC_{s-1}^{s-i-1}-sC_{s-1}^{i-1}& if~ 1\leq i \leq \frac{s}{2}\\[6pt]
sC_{s-1}^{i-1}-sC_{s-1}^{s-i-1}& if~ \frac{s}{2}< i< s \\[6pt]
sC_{s-1}^{s-1}& if~ i=s\\
\end{array}
\right.
$$
Further, with $C_{s-1}^{i-1}=\frac{i}{s}C_s^i$, it can be deduced that
$$
\sum_{i=0}^{s}(C_s^i|s-2i|)=2\lceil \frac{s}{2}\rceil C_s^{\lceil \frac{s}{2}\rceil}
$$
Then the fist result is  obtained as

$$
 \mathds{E}(f)=2\mu\sqrt{\frac{d}{s}}\frac{1}{2^s}\lceil \frac{s}{2}\rceil C_s^{\lceil \frac{s}{2}\rceil}
$$
  \item[2)]  Following the proof above, it is clear that  $ \mathds{E}(f(x))|_{s=1}=f(x)|_{s=1}=\mu\sqrt{d}$. As for $ \mathds{E}(f(x))|_{s>1}$, it is evaluated under two  cases:
  \begin{itemize}
  \item if $s$ is odd,$$
  \frac{\mathds{E}(f(x))|_{s}}{\mathds{E}(f(x))|_{s-2}}=\frac{\frac{2}{\sqrt{s}}\frac{1}{2^s} \frac{s+1}{2} C_s^{ \frac{s+1}{2}}}{\frac{2}{\sqrt{s-2}}\frac{1}{2^{s-2}} \frac{s-1}{2} C_{s-2}^{ \frac{s-1}{2}}}=\frac{\sqrt{s(s-2)}}{s-1}<1
  $$ namely, $\mathds{E}(f(x))$ decreases monotonically with respect to $s$. Clearly, in this case $\mathds{E}(f(x))|_{s=1}> \mathds{E}(f(x))|_{s>1}$;
  \item  if $s$ is even, $$
    \frac{\mathds{E}(f(x))|_{s}}{\mathds{E}(f(x))|_{s-1}}=\frac{\frac{2}{\sqrt{s}}\frac{1}{2^s} \frac{s}{2} C_s^{ \frac{s}{2}}}{\frac{2}{\sqrt{s-1}}\frac{1}{2^{s-1}} \frac{s}{2} C_{s-1}^{ \frac{s}{2}}}=\sqrt{\frac{s-1}{s}}<1
    $$ which means $\mathds{E}(f(x))|_{s=1}> \mathds{E}(f(x))|_{s>1}$, since $s-1$ is odd number for which $\mathds{E}(f(x))$  monotonically decreases.
  \end{itemize}
  Therefore the proof of the second result is completed.
  \item[3)] The proof of the third result is developed by employing  Stirling's approximation \cite{Bruijn81}$$
  s!=\sqrt{2\pi s}(\frac{s}{e})^se^{\lambda_s},~~~ 1/(12s+1)<\lambda_s<1/(12s).
  $$
  Precisely, with the formula of $\mathds{E}(f(x))$, it  can be deduced that
  \begin{itemize}
  \item if $s$ is even, $$
  \mathds{E}(f(x))=\mu\sqrt{ds}\frac{1}{2^s}\frac{s!}{\frac{s}{2}!\frac{s}{2}!}=\mu\sqrt{\frac{2d}{\pi}}e^{\lambda_s-2\lambda_{\frac{s}{2}}}
  $$
  \item if $s$ is odd, $$
\mathds{E}(f(x))=\mu\sqrt{d}\frac{s+1}{\sqrt{s}}\frac{1}{2^s}\frac{s!}{\frac{s+1}{2}!\frac{s-1}{2}!}=\mu\sqrt{\frac{2d}{\pi}}(\frac{s^2}{s^2-1})^{\frac{s}{2}}e^{\lambda_s-\lambda_{\frac{s+1}{2}}-\lambda_{\frac{s-1}{2}}}
  $$
  \end{itemize}
  Clearly $\mathop{\lim}\limits_{s\rightarrow \infty }\frac{1}{\sqrt{d}}\mathds{E}(f(x))\rightarrow \mu\sqrt{\frac{2}{\pi}}$  holds, whenever $s$ is even or odd.
\end{itemize}
\end{proof}
\section*{Appendix B.}
\label{apd-b}
\noindent{\bf Proof of Lemma \ref{lemma-4}}

\noindent
\begin{proof}
Due to the sparsity of $\mathbf{r}$ and the symmetric property of both $r_{j}$ and $z_{j}$, it is easy to derive that $f(\mathbf{r},\mathbf{z})=|\langle \mathbf{r}, \mathbf{z}\rangle|=\sqrt{\frac{d}{s}}|\sum_{j=1}^{s}z_j|$. This simplified formula will be studied in the following proof.  To present a readable proof, we first review the distribution shown in formula \eqref{eq-zi}
$$
z_j\sim\left\{
\begin{array}{lr}
N(\mu,\sigma)  & \text{with probability}~ 1/2\\
N(-\mu,\sigma)  & \text{ with probability} ~1/2\\
\end{array}
\right.
$$ where for $x\in N(\mu,\sigma)$, $\text{Pr}(x>0)=1-\epsilon$, $\epsilon=\Phi(-\frac{\mu}{\sigma})$ is a tiny positive number.  For notational simplicity, the subscript of  random variable $z_j$ is dropped in the following proof.
To ease the proof of the lemma, we first need to derive the expected value of $|x|$ with $x\sim N(\mu,\sigma^2)$:

$$
\begin{aligned}
\mathds{E}(|x|)&=\int_{-\infty}^{\infty}\frac{|x|}{\sqrt{2\pi}\sigma}e^{\frac{-(x-\mu)^2}{2\sigma^2}}dx\\
&=\int_{-\infty}^{0}\frac{-x}{\sqrt{2\pi}\sigma}e^{\frac{-(x-\mu)^2}{2\sigma^2}}dx+\int_{0}^{\infty}\frac{x}{\sqrt{2\pi}\sigma}e^{\frac{-(x-\mu)^2}{2\sigma^2}}dx\\
&=-\int_{-\infty}^{0}\frac{x-\mu}{\sqrt{2\pi}\sigma}e^{\frac{-(x-\mu)^2}{2\sigma^2}}dx+\int_{0}^{\infty}\frac{x-\mu}{\sqrt{2\pi}\sigma}e^{\frac{-(x-\mu)^2}{2\sigma^2}}dx\\
&+\mu\int_{0}^{\infty}\frac{1}{\sqrt{2\pi}\sigma}e^{\frac{-(x-\mu)^2}{2\sigma^2}}dx-\mu\int_{-\infty}^{0}\frac{1}{\sqrt{2\pi}\sigma}e^{\frac{-(x-\mu)^2}{2\sigma^2}}dx\\
&=\frac{\sigma}{\sqrt{2\pi}}e^{-\frac{(x-\mu)^2}{2\sigma^2}}|_{-\infty}^0-\frac{\sigma}{\sqrt{2\pi}}e^{-\frac{(x-\mu)^2}{2\sigma^2}}|^{\infty}_0+\mu \text{Pr}(x>0)-\mu \text{Pr}(x<0)\\
&=\sqrt{\frac{2}{\pi}}\sigma e^{-\frac{\mu^2}{2\sigma^2}}+\mu(1-2\text{Pr}(x<0))\\
&=\sqrt{\frac{2}{\pi}}\sigma e^{-\frac{\mu^2}{2\sigma^2}}+\mu(1-2\Phi(-\frac{\mu}{\sigma}))
\end{aligned}
$$
which will be used  many a time in the following proof. Then the proof of this lemma is separated into two parts as follows.
\begin{itemize}
\item[1)] This part presents the expected value of $f(r_i,z)$ for the cases $s=1$ and $s>1$.
\begin{itemize}
\item  if $s=1$, $f(\mathbf{r},\mathbf{z})=\sqrt{d}|z|$; with the  the probability density function of $z$:
$$p(z)=\frac{1}{2}\frac{1}{\sqrt{2\pi}\sigma}e^{\frac{-(z-\mu)^2}{2\sigma^2}}+\frac{1}{2}\frac{1}{\sqrt{2\pi}\sigma}e^{\frac{-(z+\mu)^2}{2\sigma^2}}$$
one can derive that
$$
\begin{aligned}
\mathds{E}(|z|)&=\int_{-\infty}^{\infty}|z|p(z)d_{z}\\
&=\frac{1}{2}\int_{-\infty}^{\infty}\frac{|z|}{\sqrt{2\pi}\sigma}e^{\frac{-(z-\mu)^2}{2\sigma^2}}dz+\frac{1}{2}\int_{-\infty}^{\infty}\frac{|z|}{\sqrt{2\pi}\sigma}e^{\frac{-(z+\mu)^2}{2\sigma^2}}dz\\
\end{aligned}
$$
with the previous result on $\mathds{E}(|x|)$, it is further deduced that  $$\mathds{E}(|z|)=\sqrt{\frac{2}{\pi}}\sigma e^{-\frac{\mu^2}{2\sigma^2}}+\mu(1-2\Phi(-\frac{\mu}{\sigma}))$$
Recall that $\Phi(-\frac{\mu}{\sigma})=\epsilon$, so
$$\mathds{E}(f)=\sqrt{d}\mathds{E}(|z|)=\sqrt{\frac{2d}{\pi}}\sigma_\mu e^{-\frac{\mu^2}{2\sigma^2}}+\mu\sqrt{d}(1-2\Phi(-\frac{\mu}{\sigma}))\approx \mu \sqrt{d} $$ if $\epsilon$ is tiny enough as illustrated in formula \eqref{eq-zi}.

\item if $s>1$, $f(\mathbf{r},\mathbf{z})=\sqrt{\frac{d}{s}}|\sum_{j=1}^{s}z|$;  let $t=\sum_{j=1}^{s}z$, then according to the symmetric distribution of $z$, $t$ holds $s+1$ different distributions:

$$
t\sim N((s-2i)\mu, s\sigma^2)~ \text{with probability} ~\frac{1}{2^s}C_s^i
$$ where ~$0\leq i\leq s$ denotes the number of $z$ drawn from $N(-\mu,\sigma^2)$. Then the PDF of $t$ can be described as
$$
p(t)=\frac{1}{2^s}\sum_{i=0}^sC_s^i\frac{1}{\sqrt{2\pi s}\sigma}e^{\frac{-(t-(s-2i)\mu)^2}{2s\sigma^2}}
$$
then,

$$
\begin{aligned}
\mathds{E}(|t|)&=\int_{-\infty}^{\infty}|t|p(t)dt\\&=\frac{1}{2^s}\sum_{i=0}^sC_s^i\int_{-\infty}^{\infty}|t|\frac{1}{\sqrt{2\pi s}\sigma}e^{\frac{-(t-(s-2i)\mu)^2}{2s\sigma^2}}dt\\
&=\frac{1}{2^s}\sum_{i=0}^sC_s^i\{\sqrt{\frac{2s}{\pi}}\sigma e^{\frac{-(s-2i)^2\mu^2}{2s\sigma^2}}+\mu|s-2i|[1-2\Phi(\frac{-|s-2i|\mu}{\sqrt{s}\sigma})]\}
\end{aligned}
$$
subsequently, the expected value of $f(r_i,z) $ can be expressed as
$$
\begin{aligned}
\mathds{E}(f)&=\mu\sqrt{\frac{d}{s}}\frac{1}{2^s}\sum_{i=0}^s(C_s^i|s-2i|)+\sigma\sqrt{\frac{2d}{\pi}}\frac{1}{2^s}\sum_{i=0}^sC_s^ie^{\frac{-(s-2i)^2\mu^2}{2s\sigma^2}}\\
&-2\mu\sqrt{\frac{d}{s}}\frac{1}{2^s}\sum_{i=0}^s[C_s^i|s-2i|\Phi(\frac{-|s-2i|\mu}{\sqrt{s}\sigma})]
\end{aligned}
$$
\end{itemize}
\item[2)] This part derives the upper bound of the aforementioned $\mathds{E}(f)|_{s>1}$. For simpler expression, the  three factors of above expression for $\mathds{E}(f)|_{s>1}$ are sequentially represented by $f_1$, $f_2$ and $f_3$, and then are analyzed, respectively.
    \begin{itemize}
    \item for $f_1=\mu\sqrt{\frac{d}{s}}\frac{1}{2^s}\sum_{i=0}^s(C_s^i|s-2i|)$, it can be rewritten as
    $$
    f_1=2\mu\sqrt{\frac{d}{s}}\frac{1}{2^s}C_s^{\lceil\frac{s}{2}\rceil}\lceil\frac{s}{2}\rceil
    $$
    \item for $f_2=\sigma\sqrt{\frac{2d}{\pi}}\frac{1}{2^s}\sum_{i=0}^sC_s^ie^{\frac{-(s-2i)^2\mu^2}{2s\sigma^2}}$, first,  we  can bound
        $$
       \left\{
        \begin{array}{ll}
        e^{\frac{-(s-2i)^2\mu^2}{2s\sigma^2}} <\text{exp}(-\frac{\mu^2}{\sigma^2}) & \text{if}~ i<\alpha ~\text{or}~ i>\alpha\\[6pt]
       e^{\frac{-(s-2i)^2\mu^2}{2s\sigma^2}}  \leq 1 & \text{if} ~\alpha\leq i\leq s-\alpha\\
        \end{array}
        \right.
        $$ where $\alpha=\lceil\frac{s-\sqrt{s}}{2}\rceil$. Take it into $f_2$,
        $$
        \begin{aligned}
        f_2&<\sigma\sqrt{\frac{2d}{\pi}}\frac{1}{2^s}\sum_{i=0}^{\alpha-1}C_s^ie^{\frac{-\mu^2}{\sigma^2}}+\sigma\sqrt{\frac{2d}{\pi}}\frac{1}{2^s}\sum_{i=s-\alpha+1}^{s}C_s^ie^{\frac{-\mu^2}{2\sigma^2}}
        +\sigma\sqrt{\frac{2d}{\pi}}\frac{1}{2^s}\sum_{i=\alpha}^{s-\alpha}C_s^i\\
        &<\sigma\sqrt{\frac{2d}{\pi}}e^{\frac{-\mu^2}{\sigma^2}}+\sigma\sqrt{\frac{2d}{\pi}}\frac{1}{2^s}\sum_{i=\alpha}^{s-\alpha}C_s^i
        \end{aligned}
        $$
        Since $C_s^i\leq C_s^{\lceil s/2\rceil}$,
        $$
        \begin{aligned}
        f_2&<\sigma\sqrt{\frac{2d}{\pi}}e^{\frac{-\mu^2}{\sigma^2}}+\sigma\sqrt{\frac{2d}{\pi}}\frac{1}{2^s}(\lfloor \sqrt{s}\rfloor+1) C_s^{\lceil s/2\rceil}\\
        &\leq \sigma\sqrt{\frac{2d}{\pi}}e^{\frac{-\mu^2}{\sigma^2}}+\sigma\sqrt{\frac{2d}{\pi}}\frac{1}{2^s} \sqrt{s} C_s^{\lceil s/2\rceil}+\sigma\sqrt{\frac{2d}{\pi}}\frac{1}{2^s}  C_s^{\lceil s/2\rceil}\\
        &\leq \sigma\sqrt{\frac{2d}{\pi}}e^{\frac{-\mu^2}{\sigma^2}}+\sigma\sqrt{\frac{2d}{\pi}}\frac{1}{2^s}\frac{2} {\sqrt{s}} C_s^{\lceil s/2\rceil}{\lceil \frac{s}{2}\rceil}+\sigma\sqrt{\frac{2d}{\pi}}\frac{1}{2^s}  C_s^{\lceil s/2\rceil}
        \end{aligned}
        $$
        with Stirling's approximation,
        $$
        f_2<\left\{
        \begin{array}{ll}
        \sqrt{\frac{2d}{\pi}}\sigma e^{\frac{-\mu^2}{2\sigma^2}}+\sqrt{d}\frac{2}{\pi}\sigma e^{\lambda_s-2\lambda_{s/2}}+\sqrt{\frac{d}{s}}\frac{2}{\pi}\sigma e^{\lambda_s-2\lambda_{s/2}}&\text{if $s$ is even}\\[10pt]
         \sqrt{\frac{2d}{\pi}}\sigma e^{\frac{-\mu^2}{2\sigma^2}}+\sqrt{{d}}\frac{2\sigma}{\pi}(\frac{s^2}{s^2-1})^{\frac{s}{2}}e^{\lambda_s-\lambda_{\frac{s+1}{2}}-\lambda_{\frac{s-1}{2}}}\\
         +\sqrt{{d}}\frac{2\sigma}{\pi}\frac{\sqrt{s}}{s+1}(\frac{s^2}{s^2-1})^{\frac{s}{2}}e^{\lambda_s-\lambda_{\frac{s+1}{2}}-\lambda_{\frac{s-1}{2}}}&\text{if $s$ is odd}
        \end{array}
        \right.
        $$
        \item for $f_3=-2\mu\sqrt{\frac{d}{s}}\frac{1}{2^s}\sum_{i=0}^s[C_s^i|s-2i|\Phi(\frac{-|s-2i|\mu}{\sqrt{s}\sigma})]$, with the previous defined $\alpha$,
        $$
        \begin{aligned}
        f_3&\leq -2\mu\sqrt{\frac{d}{s}}\frac{1}{2^s}\sum_{i=\alpha}^{s-\alpha}[C_s^i|s-2i|\Phi(\frac{-|s-2i|\mu}{\sqrt{s}\sigma})]\\
        &\leq  -2\mu\sqrt{\frac{d}{s}}\frac{1}{2^s}\sum_{i=\alpha}^{s-\alpha}[C_s^i|s-2i|\Phi(\frac{-\mu}{\sigma})]\\
        &=  -2\mu\epsilon\sqrt{\frac{d}{s}}\frac{1}{2^s}\sum_{i=\alpha}^{s-\alpha}[C_s^i|s-2i|]\\
        &= -2\mu\epsilon\sqrt{\frac{d}{s}}\frac{1}{2^s}(2sC_{s-1}^{\lceil{\frac{s}{2}-1}\rceil}-2sC_{s-1}^{\alpha-1})\\
        &= -4\mu\epsilon\sqrt{ds}\frac{1}{2^s}(C_{s-1}^{\lceil{\frac{s}{2}-1}\rceil}-C_{s-1}^{\alpha-1})\\
        &\leq 0
        \end{aligned}
        $$
       \end{itemize}
        finally, we can further deduce that
        $$
        \begin{aligned}
        &\mathds{E}(f)|_{s>1}=f_1+f_2+f_3\\[6pt]
         &<\left\{
         \begin{array}{ll}
        2\mu\frac{1}{2^s}\sqrt{\frac{d}{s}}C_s^{\lceil \frac{s}{2} \rceil} +\frac{2\sigma}{\pi}\sqrt{d}e^{\lambda_s-2\lambda_{\frac{s}{2}}}+\sqrt{\frac{2d}{\pi}}\sigma e^{\frac{-\mu^2}{2\sigma^2}}+\sqrt{\frac{d}{s}}\frac{2}{\pi}\sigma e^{\lambda_s-2\lambda_{s/2}}\\-4\mu\epsilon\sqrt{ds}\frac{1}{2^s}(C_{s-1}^{\lceil{\frac{s}{2}-1}\rceil}-C_{s-1}^{\alpha-1}) & \text{if $s$ is even}\\[12pt]
        2\mu\frac{1}{2^s}\sqrt{\frac{d}{s}}C_s^{\lceil \frac{s}{2} \rceil}+\frac{2\sigma}{\pi}\sqrt{d}\frac{s^2}{s^2-1}^{\frac{s}{2}}e^{\lambda_s-\lambda_{\frac{s+1}{2}}-\lambda_{\frac{s-1}{2}}}+\sqrt{\frac{2d}{\pi}}\sigma e^{\frac{-\mu^2}{2\sigma^2}}\\+\sqrt{{d}}\frac{2\sigma}{\pi}\frac{\sqrt{s}}{s+1}(\frac{s^2}{s^2-1})^{\frac{s}{2}}e^{\lambda_s-\lambda_{\frac{s+1}{2}}-\lambda_{\frac{s-1}{2}}}-4\mu\epsilon\sqrt{ds}\frac{1}{2^s}(C_{s-1}^{\lceil{\frac{s}{2}-1}\rceil}-C_{s-1}^{\alpha-1}) & \text{if $s$ is odd}
         \end{array}
        \right.\\[10pt]
        &=\left\{
        \begin{array}{ll}
        (\sqrt{\frac{2d}{\pi}}\mu+\frac{4\sigma}{\pi}\sqrt{d})e^{\lambda_s-2\lambda_{\frac{s}{2}}}+\sqrt{\frac{2d}{\pi}}\sigma e^{\frac{-\mu^2}{2\sigma^2}}+\sqrt{\frac{d}{s}}\frac{2}{\pi}\sigma e^{\lambda_s-2\lambda_{s/2}}\\-4\mu\epsilon\sqrt{ds}\frac{1}{2^s}(C_{s-1}^{\lceil{\frac{s}{2}-1}\rceil}-C_{s-1}^{\alpha-1}) & \text{if $s$ is even}\\[12pt]
        (\sqrt{\frac{2d}{\pi}}\mu+\frac{4\sigma}{\pi}\sqrt{d})(\frac{s^2}{s^2-1})^{\frac{s}{2}}e^{\lambda_s-\lambda_{\frac{s+1}{2}}-\lambda_{\frac{s-1}{2}}}+\sqrt{\frac{2d}{\pi}}\sigma e^{\frac{-\mu^2}{2\sigma^2}}\\+\sqrt{{d}}\frac{2\sigma}{\pi}\frac{\sqrt{s}}{s+1}(\frac{s^2}{s^2-1})^{\frac{s}{2}}e^{\lambda_s-\lambda_{\frac{s+1}{2}}-\lambda_{\frac{s-1}{2}}}-4\mu\epsilon\sqrt{ds}\frac{1}{2^s}(C_{s-1}^{\lceil{\frac{s}{2}-1}\rceil}-C_{s-1}^{\alpha-1}) & \text{if $s$ is odd}
        \end{array}
        \right.
        \end{aligned}
        $$
      \item[3)] This part discusses the condition for $$\mathds{E}(f)|_{s>1}<\mathds{E}(f)|_{s=1}=\sqrt{\frac{2d}{\pi}}\sigma e^{-\frac{\mu^2}{2\sigma^2}}+\mu\sqrt{d}(1-2\Phi(-\frac{\mu}{\sigma}))$$ by further relaxing the upper bound of $\mathds{E}(f)|_{s>1}$.
          \begin{itemize}
          \item  if $s$ is even, since $f_3\leq 0$,
          $$
          \begin{aligned}
          \mathds{E}(f)|_{s>1}&<(\sqrt{\frac{2d}{\pi}}\mu+\frac{2\sigma}{\pi}\sqrt{d})e^{\lambda_s-2\lambda_{\frac{s}{2}}}+\sqrt{\frac{d}{s}}\frac{2}{\pi}\sigma e^{\lambda_s-2\lambda_{s/2}}+\sqrt{\frac{2d}{\pi}}\sigma e^{\frac{-\mu^2}{2\sigma^2}}\\
          &\leq (\sqrt{\frac{2d}{\pi}}\mu+\frac{2\sigma}{\pi}\sqrt{d})+\sqrt{\frac{d}{s}}\frac{2}{\pi}\sigma +\sqrt{\frac{2d}{\pi}}\sigma e^{\frac{-\mu^2}{2\sigma^2}}\\
          &=\mu\sqrt{d}(\sqrt{\frac{2}{\pi}}+(1+\frac{1}{\sqrt{s}})\frac{2\sigma}{\pi\mu})+\sqrt{\frac{2d}{\pi}}\sigma e^{\frac{-\mu^2}{2\sigma^2}}
          \end{aligned}
          $$ Clearly $\mathds{E}(f)|_{s>1}<\mathds{E}(f)|_{s=1}$, if $\sqrt{\frac{2}{\pi}}+(1+\frac{1}{\sqrt{2}})\frac{2\sigma}{\pi\mu}\leq 1-2\Phi(-\frac{\mu}{\sigma})$. This condition is well satisfied when $\mu>>\sigma$, since $\Phi(-\frac{\mu}{\sigma})$ decreases monotonically with increasing $\mu/\sigma$.
          \item  if $s$ is odd, with $f_3\leq 0$,
            $$
          \begin{aligned}
          \mathds{E}(f)|_{s>1}&<(\sqrt{\frac{2d}{\pi}}\mu+\frac{2\sigma}{\pi}\sqrt{d})(\frac{s^2}{s^2-1})^{\frac{s}{2}}+\sqrt{{d}}\frac{2\sigma}{\pi}\frac{\sqrt{s}}{s+1}(\frac{s^2}{s^2-1})^{\frac{s}{2}}+\sqrt{\frac{2d}{\pi}}\sigma e^{\frac{-\mu^2}{2\sigma^2}}
          \end{aligned}
          $$  It can be proved that $(\frac{s^2}{s^2-1})^{\frac{s}{2}}$ decreases monotonically with respect to $s$. This yields that

          $$
          \begin{aligned}
          \mathds{E}(f)|_{s>1}<(\sqrt{\frac{2d}{\pi}}\mu+(1+\frac{\sqrt{3}}{4})\frac{2\sigma}{\pi}\sqrt{d})(\frac{3^2}{3^2-1})^{\frac{3}{2}}+\sqrt{\frac{2d}{\pi}}\sigma e^{\frac{-\mu^2}{2\sigma^2}}
          \end{aligned}
          $$ in this case $\mathds{E}(f)|_{s>1}<\mathds{E}(f)|_{s=1}$, if $(\frac{9}{8})^{\frac{3}{2}}(\sqrt{\frac{2}{\pi}}+(1+\frac{\sqrt{3}}{4})\frac{2\sigma}{\pi\mu})\leq 1-2\Phi(-\frac{\mu}{\sigma})$.
          \end{itemize}
            Summarizing above two cases for $s$ , finally
          $$
          \mathds{E}(f)|_{s>1}<\mathds{E}(f)|_{s=1},~ \text{if} ~ (\frac{9}{8})^{\frac{3}{2}}[\sqrt{\frac{2}{\pi}}+(1+\frac{\sqrt{3}}{4})\frac{2}{\pi}(\frac{\mu}{\sigma})^{-1}]+2\Phi(-\frac{\mu}{\sigma})\leq 1
          $$
\end{itemize}
\end{proof}


\section*{Appendix C.}
\label{apd-c}
\noindent{\bf Proof of Lemma \ref{lemma-5}}

\vspace{3pt}
\noindent
First, one can rewrite $f(\mathbf{r},\mathbf{z})=|\Sigma_{j=1}^{j=d}(r_{j}z_j)|=\mu |x|$, where $x\in N(0,d)$, since i.i.d $r_{j}\in N(0,1)$ and $z_j \in \{\pm \mu\}$ with equal probability. Then one can prove that
$$
\begin{aligned}
\mathds{E}(|x|)&=\int_{-\infty}^{0}\frac{-x}{\sqrt{2\pi d}}e^{-\frac{x^2}{2d}}dx+\int_{0}^{\infty}\frac{x}{\sqrt{2\pi d}}e^{-\frac{x^2}{2d}}dx\\&=2\int_0^{\infty}\frac{\sqrt{d}}{\sqrt{2\pi}}e^{-\frac{x^2}{2d}}d\frac{x^2}{2d}\\
&=2\sqrt{d}\int_0^{\infty}\frac{1}{\sqrt{2\pi}}e^{-\alpha}d\alpha\\
&=\sqrt{\frac{2d}{\pi}}
\end{aligned}
$$
Finally, it is derived that $\mathds{E}(f)=\mu\mathds{E}(|x|)=\mu\sqrt{\frac{2d}{\pi}}$.

\bibliographystyle{IEEEtran}
\vskip 0.2in
\bibliography{egbib}

\end{document}